\documentclass[10pt,twocolumn,letterpaper]{article}
\pdfoutput=1

\usepackage[absolute]{textpos}
\usepackage{wacv}
\usepackage{times}
\usepackage{epsfig}
\usepackage{graphicx}
\usepackage{amsmath}
\usepackage{amssymb}

\usepackage{color}
\usepackage{textcomp}
\usepackage{amsthm}
\usepackage{enumitem}
\usepackage{tabularx}
\usepackage{multirow}
\usepackage{booktabs}
\usepackage{subcaption}
\usepackage{xcolor,colortbl}
\usepackage[page,title]{appendix}
\usepackage{algorithm} 
\usepackage{algorithmic} 
\usepackage[colorlinks,bookmarks=false]{hyperref} 
\usepackage[capitalize]{cleveref} 

\renewcommand{\cf}{\textit{cf}\onedot}

\newcommand{\semsim}{s_\mathcal{G}}
\newcommand{\semdist}{d_\mathcal{G}}

\DeclareMathOperator{\lcs}{lcs}
\DeclareMathOperator{\treeheight}{height}
\DeclareMathOperator{\rootnode}{root}

\newtheorem{theorem}{Theorem}

\crefformat{equation}{(#2#1#3)}
\Crefname{appsec}{Appendix}{Appendices}
\crefname{appsec}{Appendix}{Appendices}
\definecolor{fsuBlue}{RGB}{135,178,231}

\wacvfinalcopy 

\def\wacvPaperID{317} 

\ifwacvfinal\fi
\begin{document}
    
\title{Hierarchy-based Image Embeddings for Semantic Image Retrieval}

\author{Bj{\"o}rn Barz \hspace{3cm} Joachim Denzler \\
     \\
    Friedrich Schiller University Jena\\
    Computer Vision Group\\
    {\tt\small bjoern.barz@uni-jena.de}
}

\maketitle
\ifwacvfinal\thispagestyle{empty}\fi

\begin{textblock*}{\textwidth}(18mm,12mm)
    \textblockcolour{fsuBlue}
    \vspace{2mm}
    \tiny
    \centering
    Bj{\"o}rn Barz and Joachim Denzler.\\
    ``Hierarchy-based Image Embeddings for Semantic Image Retrieval.''\\
    \textit{IEEE Winter Conference on Applications of Computer Vision (WACV) 2019.}\\
    \copyright\ 2019 IEEE. Personal use of this material is permitted.
    Permission from IEEE must be obtained for all other uses, in any current or future media, including reprinting/republishing this material for advertising or promotional purposes, creating new collective works, for resale or redistribution to servers or lists, or reuse of any copyrighted component of this work in other works.
    The final publication will be available at
    \href{https://ieeexplore.ieee.org/abstract/document/8658633}{ieeexplore.ieee.org}.\\
    \vspace{2mm}
\end{textblock*}

\begin{abstract}
Deep neural networks trained for classification have been found to learn powerful image representations, which are also often used for other tasks such as comparing images w.r.t.\ their visual similarity.
However, visual similarity does not imply semantic similarity.
In order to learn semantically discriminative features, we propose to map images onto class embeddings whose pair-wise dot products correspond to a measure of semantic similarity between classes.
Such an embedding does not only improve image retrieval results, but could also facilitate integrating semantics for other tasks, e.g., novelty detection or few-shot learning.
We introduce a deterministic algorithm for computing the class centroids directly based on prior world-knowledge encoded in a hierarchy of classes such as WordNet.
Experiments on CIFAR-100, NABirds, and ImageNet show that our learned semantic image embeddings improve the semantic consistency of image retrieval results by a large margin.

\textbf{Source code:} \url{https://github.com/cvjena/semantic-embeddings}
\end{abstract}

\section{Introduction}
\label{sec:introduction}

\begin{figure*}[t]
    \includegraphics[width=\linewidth]{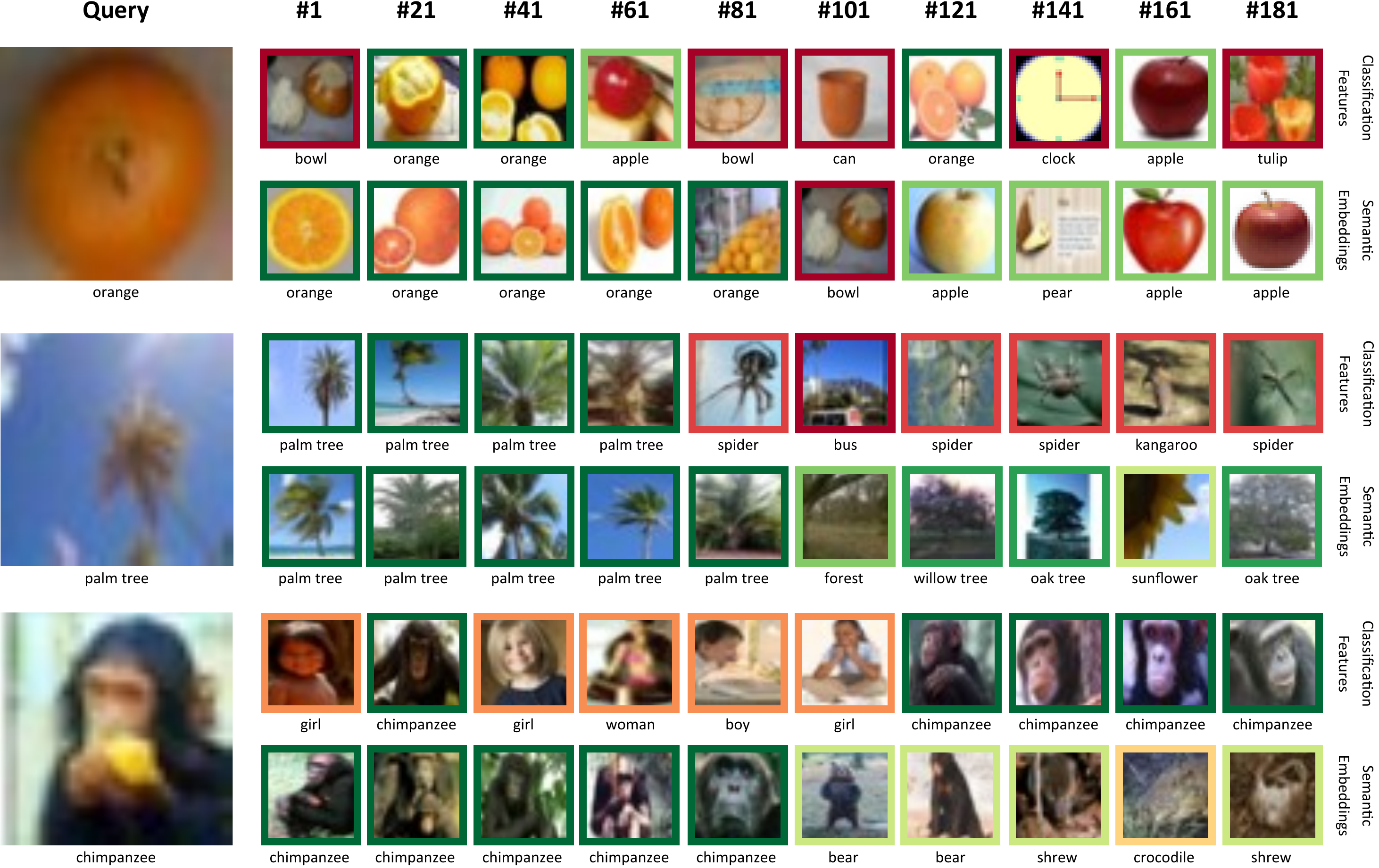}%
    \caption{Comparison of image retrieval results on CIFAR-100 \cite{krizhevsky2009cifar} for 3 exemplary queries using features extracted from a variant of ResNet-110 \cite{he2016resnet} trained for classification and semantic embeddings learned by our method. The border colors of the retrieved images correspond to the semantic similarity between their class and the class of the query image (with dark green being most similar and dark red being most dissimilar). It can be seen that hierarchy-based semantic image embeddings lead to much more semantically consistent retrieval results.}
    \label{fig:retrieval_comparison}
\end{figure*}

\noindent
During the past few years, deep convolutional neural networks (CNNs) have continuously advanced the state-of-the-art in image classification \cite{krizhevsky2012alexnet,simonyan2014vgg,he2016resnet,huang2016densenet}
and many other tasks.
The intermediate image representations learned by such CNNs trained for classification have also proven to be powerful image descriptors for retrieving images from a database that are visually or semantically similar to one or more query images given by the user \cite{babenko2014neural,sharif2014cnn,babenko2015aggregating}.
This task is called content-based image retrieval (CBIR) \cite{smeulders2000cbir}.

Usually, the categorical cross-entropy loss after a softmax activation is used for training CNN classifiers, which results in well-separable features.
However, these features are not necessarily discriminative, \ie, the inter-class variance of the learned representations may be small compared to the intra-class variance \cite{wen2016centerloss}.
The average distance between the classes ``cat'' and ``dog'' in feature space may well be equally large as the distance between ``cat'' and ``forest''.
For image retrieval, that situation is far from optimal, since the nearest neighbors of a certain image in feature space may belong to completely different classes.
This can, for instance, be observed in the upper row of each example in \cref{fig:retrieval_comparison}, where we used features extracted from the global average pooling layer of ResNet-110 \cite{he2016resnet} for image retrieval on the CIFAR-100 dataset \cite{krizhevsky2009cifar}.
While the first results often belong to the same class as the query, the semantic relatedness of the results at later positions deteriorates significantly.
Those seem to be mainly visually similar to the query with respect to shape and color, but not semantically.

Many authors have therefore proposed a variety of metric learning losses, aiming to increase the separation of classes, while minimizing the distances between samples from the same class.
Popular instances of this branch of research are the contrastive loss \cite{chopra2005contrastiveloss}, the triplet loss \cite{schroff2015facenet}, and the quadruplet loss \cite{chen2017quadrupletloss}.
However, these methods require sampling of hard pairs, triplets, or even quadruplets of images, making training cumbersome and expensive.
On the other hand, they still do not impose any constraints on inter-class relationships:
Though images of the same class should be tightly clustered together in feature space, neighboring clusters may still be completely unrelated.

This harms the semantic consistency of image retrieval results.
Consider, for example, a query image showing a poodle.
It is unclear, whether the user is searching for images of other poodles only, for images of any other dog breed, or even for images of animals in general.
Current CBIR methods would---ideally---first retrieve all images of the same class as the query but then continue with mostly unrelated, only \textit{visually} similar images from other classes.

In this work, we propose a method for learning \textit{semantically} meaningful image representations, so that the Euclidean distance in the feature space directly corresponds to the semantic dissimilarity of the classes.
Since ``semantic similarity'' is, in principle, arbitrary, we rely on prior knowledge about the physical world encoded in a hierarchy of classes.
Such hierarchies are readily available in many domains, since various disciplines have been striving towards organizing the entities of the world in ontologies for years.
WordNet \cite{fellbaum1998wordnet}, for example, is well-known for its good coverage of the world with over 80,000 semantic concepts and the Wikispecies project\footnote{\url{https://species.wikimedia.org/}} provides a taxonomy of living things comprising more than half a million nodes.

We make use of the knowledge explicitly encoded in such hierarchies to derive a measure of semantic similarity between classes and introduce an algorithm for explicitly computing target locations for all classes on a unit hypersphere, so that the pair-wise dot products of their embeddings equal their semantic similarities.
We then learn a transformation from the space of color images to this semantically meaningful feature space, so that the correlation between image features and their class-specific target embedding is maximized.
This can easily be done using the negated dot product as a loss function, without having to consider pairs or triplets of images or hard-negative mining.

Exemplary retrieval results of our system are shown in the bottom rows of each example in \cref{fig:retrieval_comparison}.
It can be seen that our semantic class embeddings lead to image features that are much more invariant against superficial visual differences:
An orange is semantically more similar to a green apple than to orange tulips and a palm tree is more similar to an oak than to a spider.
The results follow the desired scheme described above: For a query image showing an orange, all oranges are retrieved first, then all apples, pears, and other fruits.
As a last example, incorporating semantic information about classes successfully helps avoiding the not so uncommon mistake of confusing humans with apes.

We discuss related work on learning semantic image representations in \cref{sec:related-work}. Our approach for computing class embeddings based on a class hierarchy is presented in \cref{sec:class-embeddings}, and \cref{sec:learning-embeddings} explains how we learn to map images onto those embeddings. Experimental results on three popular datasets are presented in \cref{sec:experiments}, and \cref{sec:conclusions} concludes this work.

\section{Related Work}
\label{sec:related-work}

\noindent
Since the release of the ImageNet dataset \cite{deng2009imagenet} in 2009, many authors have proposed to leverage the WordNet ontology, which the classes of ImageNet have been derived from, for improving classification and image retrieval:
The creators of ImageNet, Deng et al.~\cite{deng2011hierarchical}, derived a deterministic bilinear similarity measure from the taxonomy of classes for comparing image feature vectors composed of class-wise probability predictions. This way, images assigned to different classes can still be considered as similar if the two classes are similar to each other.

Regarding classification, Zhao et al.~\cite{zhao2011large} modify multinomial logistic regression to take the class structure into account, using the dissimilarity of classes as misclassification cost.
Verma et al.~\cite{verma2012learning}, on the other hand, learn a specialized Mahalanobis distance metric for each node in a class hierarchy and combine them along the paths from the root to the leafs. This results in different metrics for all classes and only allows for nearest-neighbor-based classification methods.
In contrast, Chang et al.~\cite{chang2015large} learn a global Mahalanobis metric on the space of class-wise probability predictions, where they enforce margins between classes proportional to their semantic dissimilarities.
HD-CNN \cite{yan2015hd} follows an end-to-end learning approach by dividing a CNN into a coarse and several fine classifiers based on two levels of a class hierarchy, fusing predictions at the end.

All these approaches exploit the class structure for classification or retrieval, but only at the classifier level instead of the features themselves.
Our approach, in contrast, embeds images into a semantically meaningful space where the dot product corresponds to the similarity of classes.
This does not only make semantic image retrieval straightforward, but also enables the application of a wide variety of existing methods that rely on metric feature spaces, \eg, clustering or integration of relevance feedback into the retrieval \cite{deselaers2008learning}.

A very similar approach has been taken by Weinberger et al.~\cite{weinberger2009taxem}, who propose ``taxonomy embeddings'' (``taxem'') for categorization of documents.
However, they do not specify how they obtain semantic class similarities from the taxonomy.
Moreover, they learn a linear transformation from hand-crafted document features onto the class embedding space using ridge regression, whereas we perform end-to-end learning using neural networks.

More recently, several authors proposed to jointly learn embeddings of classes and images based solely on visual information, \eg, using the ``center loss'' \cite{wen2016centerloss} or a ``label embedding network'' \cite{sun2017label}.
However, semantics are often too complex for being derived from visual information only.
For example, the label embedding network \cite{sun2017label} learned that pears are similar to bottles, because their shape and color is often similar and the image information alone is not sufficient for learning that fruits and man-made containers are fundamentally different concepts.

To avoid such issues, Frome et al. (``DeViSE'' \cite{frome2013devise}) and Li et al.~\cite{li2017learning} propose to incorporate prior world-knowledge by mapping images onto word embeddings of class labels learned from text corpora \cite{mikolov2013distributed,pennington2014glove}.
To this end, Frome et al.~\cite{frome2013devise} need to pre-train their image embeddings for classification initially, and Li et al.~\cite{li2017learning} first perform region mining and then use three sophisticated loss functions, requiring the mining of either hard pairs or triplets.

In contrast to this expensive training procedure relying on the additional input of huge text corpora, we show how to explicitly construct class embeddings based on prior knowledge encoded in an easily obtainable hierarchy of classes, without the need to learn such embeddings approximately.
These embeddings also allow for straightforward learning of image representations by simply maximizing the dot product of image and class embeddings.

A broader overview of research aiming for incorporating prior knowledge into deep learning of image representations can, for instance, be found in \cite{setti2018knowandlearn}.

\section{Hierarchy-based Class Embeddings}
\label{sec:class-embeddings}

\noindent
In the following, we first describe how we measure semantic similarity between classes based on a hierarchy and then introduce our method for explicitly computing class embeddings based on those pair-wise class similarities.

\subsection{Measuring Semantic Similarity}
\label{subsec:semantic-similarity}

\noindent
Let $\mathcal{G} = (V, E)$ be a directed acyclic graph with nodes $V$ and edges $E \subseteq V \times V$, specifying the hyponymy relation between semantic concepts.
In other words, an edge $(u, v) \in E$ means that $v$ is a sub-class of $u$.
The actual classes of interest $\mathcal{C} = \{c_1, \dotsc, c_n\} \subseteq V$ are a subset of the semantic concepts. 
An example for such a graph, with the special property of being a tree, is given in \cref{subfig:toy-hierarchy}.

A commonly used measure for the dissimilarity $\semdist: \mathcal{C} \times \mathcal{C} \rightarrow \mathbb{R}$ of classes organized in this way is the height of the sub-tree rooted at the \textit{lowest common subsumer (LCS)} of two classes, divided by the height of the hierarchy \cite{deng2011hierarchical,verma2012learning}:
\begin{equation}
    \label{eq:lcs-height}
    \semdist(u, v) = \frac{\treeheight(\lcs(u,v))}{\max_{w \in V} \treeheight(w)} \;,
\end{equation}
where the height of a node is defined as the length of the longest path from that node to a leaf.
The LCS of two nodes is the ancestor of both nodes that does not have any child being an ancestor of both nodes as well.

Since $\semdist$ is bounded between 0 and 1, we can easily derive a measure for semantic similarity between semantic concepts as well:
\begin{equation}
    \label{eq:class-dissimilarity}
    \semsim(u, v) = 1 - \semdist(u, v) \;.
\end{equation}

For example, the toy hierarchy in \cref{subfig:toy-hierarchy} has a total height of 3, the LCS of the classes ``dog'' and ``cat'' is ``mammal'' and the LCS of ``dog'' and ``trout'' is ``animal''. It follows that $\semdist(\mathrm{``dog``}, \mathrm{``cat``}) = \frac{1}{3}$ and $\semdist(\mathrm{``dog``}, \mathrm{``trout``}) = \frac{2}{3}$.

Note that though $\semdist$ is symmetric and non-negative, it is only guaranteed to be a proper metric if $\mathcal{G}$ is a tree with all classes of interest being leaf nodes (the proof can be found in \cref{app:metric-proof}).
Some well-known ontologies such as WordNet \cite{fellbaum1998wordnet} do not have this property and violate the triangle inequality.
For instance, the WordNet synset ``golfcart'' is a hypernym of both ``vehicle'' and ``golf equipment''. It is hence similar to cars and golf balls, while both are not similar to each other at all.

Our goal is to embed classes onto a unit hypersphere so that their dot product corresponds to $\semsim$.
Thus, the Euclidean distance between such embeddings of the classes equals $\sqrt{2 \cdot \semdist}$, which hence has to be a metric.
Therefore, we assume the hierarchy $\mathcal{G}$ to be given as tree in the following.
In case that this assumption does not hold, approaches from the literature for deriving tree-shaped hierarchies from ontologies such as WordNet could be employed.
For instance, YOLO-9000 \cite{redmon2016yolo9000} starts with a tree consisting of the root-paths of all concepts which have only one such path and then successively adds paths to other concepts that result in the least number of nodes added to the existing tree.

\begin{figure}[t]
    \begin{subfigure}{.48\linewidth}%
        \centering
        \includegraphics[height=3.2cm]{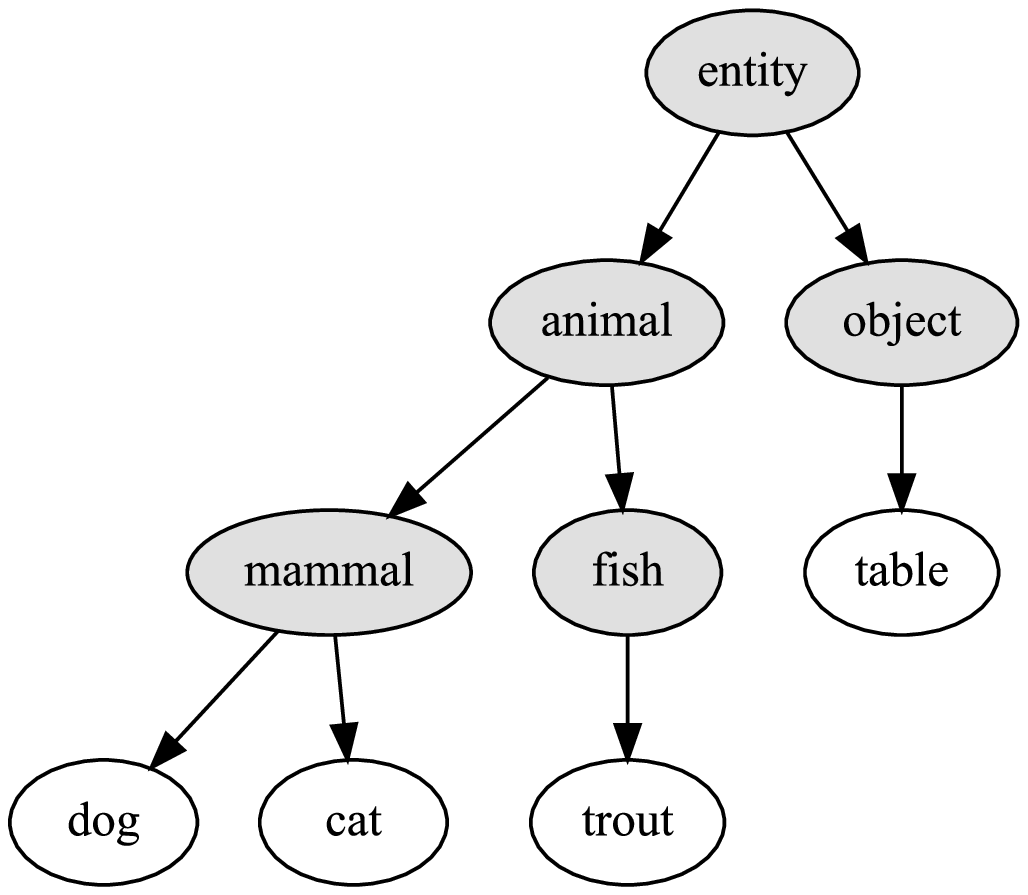}%
        \caption{}%
        \label{subfig:toy-hierarchy}%
    \end{subfigure}%
    \hfill%
    \begin{subfigure}{.48\linewidth}%
        \centering
        \includegraphics[height=3.2cm]{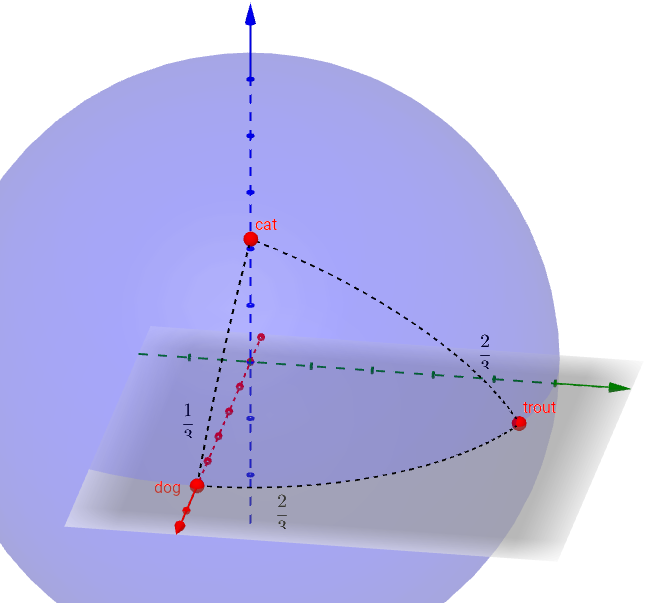}%
        \caption{}%
        \label{subfig:toy-embedding-3}%
    \end{subfigure}
    \caption{\subref{subfig:toy-hierarchy} A toy hierarchy and \subref{subfig:toy-embedding-3} an embedding of 3 classes from this hierarchy with their pair-wise $\semdist$.}
    \label{fig:toy-example}
\end{figure}

\subsection{Class Embedding Algorithm}
\label{subsec:embedding-algorithm}

\noindent
Consider $n$ classes $\mathcal{C} = \{c_1, \dotsc, c_{n}\}$ embedded in a hierarchy $\mathcal{G}$ as above.
Our aim is to compute embeddings $\varphi(c_i) \in \mathbb{R}^n$ of all classes $c_i$, $i=1,\dotsc,n$, so that
\begin{eqnarray}
    \label{eq:embedding-constraint}
    \forall_{1 \leq i, j \leq n}: && \varphi(c_i)^\top \varphi(c_j) = \semsim(c_i, c_j) \;, \\
    \label{eq:embedding-constraint-norm}
    \forall_{1 \leq i \leq n}: && \|\varphi(c_i)\| = 1 \;.
\end{eqnarray}
In other words, the correlation of class embeddings should equal the semantic similarity of the classes and all embeddings should be L2-normalized.
Eq. \eqref{eq:embedding-constraint-norm} actually is a direct consequence of \eqref{eq:embedding-constraint} in combination with the fact that $\semsim(u,u)=1$ for any class $u \in \mathcal{C}$, but we formulate it as an explicit constraint here for clarity, emphasizing that all class embeddings lie on a unit hypersphere.
This does not only allow us to use the negated dot product as a substitute for the Euclidean distance, but also accounts for the fact that L2 normalization has proven beneficial for CBIR in general, because the direction of high-dimensional feature vectors often carries more information than their magnitude \cite{jegou2014triangulation,husain2017rvd,horiguchi2017significance}.

\begin{algorithm}[t]
    \caption{Hierarchy-based Class Embeddings}
    \label{algo:class-embeddings}
    \begin{algorithmic}[1]
        
        \REQUIRE hierarchy $\mathcal{G} = (V, E), E \subseteq (V \times V)$, set $\mathcal{C} = (c_1, \dotsc, c_n) \subseteq V$ of $n$ classes
        \ENSURE embeddings $\varphi(c_i) \in \mathbb{R}^n$ for $i = 1, \dotsc, n$
        
        \STATE $\varphi(c_1) \leftarrow [1, 0, \dotsc, 0]^\top$
        \FOR{$i=2$ \TO $n$}
        \STATE $\hat{\varphi}(c_i) \leftarrow$ solution of \cref{eq:embedding-eqs}
        \STATE $x \leftarrow$ maximum of the solutions of \cref{eq:embedding-normalization}
        \STATE $\varphi(c_i) \leftarrow [\hat{\varphi}(c_i)^\top, x, 0, \dotsc, 0]^\top$
        \ENDFOR
        
    \end{algorithmic}
\end{algorithm}

We follow a step-wise approach for computing the embeddings $\varphi(c_i), i=1,\dotsc,n$, as outlined in \cref{algo:class-embeddings}:
We can always choose an arbitrary point on the unit sphere in an $n$-dimensional space as embedding for the first class $c_1$, because the constraint \eqref{eq:embedding-constraint} is invariant against arbitrary rotations and translations.
Here, we choose $\varphi(c_1) = [1, 0, \dotsc, 0]^\top$, which will simplify further calculations.

The remaining classes $c_i$, $i=2,\dotsc,n$, are then placed successively so that each new embedding has the correct dot product with any other already embedded class:
\begin{equation}
    \label{eq:embedding-eqs}
    \forall_{1 \leq j < i}: \varphi(c_j)^\top \varphi(c_i) = \semsim(c_j, c_i) \;.
\end{equation}

This is a system of $i-1$ linear equations, where $\varphi(c_i)$ is the vector of unknown variables.
As per our construction, only the first $j$ dimensions of all $\varphi(c_j)$, $j=1,\dotsc,i-1$, are non-zero, so that the effective number of free variables is $i-1$.
The system is hence well-determined and in lower-triangular form, so that it has a unique solution that can be computed efficiently with $\mathcal{O}(i^2)$ floating-point operations using forward substitution.

This solution $\hat{\varphi}(c_i) = [\varphi(c_i)_1, \dotsc, \varphi(c_i)_{i-1}]^\top$ for the first $i-1$ coordinates of $\varphi(c_i)$ already fulfills \eqref{eq:embedding-constraint}, but not \eqref{eq:embedding-constraint-norm}, \ie, it is not L2-normalized.
Thus, we use an additional dimension to achieve normalization:
\begin{equation}
    \label{eq:embedding-normalization}
    \varphi(c_i)_i = \sqrt{1 - \|\hat{\varphi}(c_i)\|^2} \;.
\end{equation}
Without loss of generality, we always choose the non-negative solution of this equation, so that all class embeddings lie entirely in the positive orthant of the feature space.

Due to this construction, exactly $n$ feature dimensions are required for computing a hierarchy-based embedding of $n$ classes.
An example of such an embedding for 3 classes is given in \cref{subfig:toy-embedding-3}.
The overall complexity of the algorithm outlined in \cref{algo:class-embeddings} is $\mathcal{O}(n^3)$.

\subsection{Low-dimensional Approximation}
\label{subsec:low-dimensional-embeddings}

\noindent
In settings with several thousands of classes, the number of features required by our algorithm from \cref{subsec:embedding-algorithm} might become infeasible.
However, it is possible to obtain class embeddings of arbitrary dimensionality whose pair-wise dot products best \textit{approximate} the class similarities.

Let $\mathcal{S} \in \mathbb{R}^{n \times n}$ be the matrix of pair-wise class similarities, \ie, $S_{ij} = \semsim(c_i, c_j), 1 \leq i,j \leq n,$ and $\Phi \in \mathbb{R}^{n \times n}$ be a matrix whose $i$-th row is the embedding $\varphi(c_i)$ for class $c_i$.
Then, we can reformulate \eqref{eq:embedding-constraint} as
\begin{equation}
    \label{eq:embedding-eigendecomposition}
    \Phi \cdot \Phi^\top = \mathcal{S} = Q \Lambda Q^\top \;,
\end{equation}
where $Q \in \mathbb{R}^{n \times n}$ is a matrix whose rows contain the eigenvectors of $\mathcal{S}$ and $\Lambda \in \mathbb{R}^{n \times n}$ is a diagonal matrix containing the corresponding eigenvalues.

Thus, we could also use the eigendecomposition of $\mathcal{S}$ to obtain the class embeddings as $\Phi = Q \Lambda^{1/2}$.
However, we have found our algorithm presented in \cref{subsec:embedding-algorithm} to provide better numerical accuracy, resulting in a maximum error of pairwise distances of $1.7 \times 10^{-15}$ for the 1000 classes of ILSVRC \cite{ILSVRC15}, while the eigendecomposition only obtains $1.7 \times 10^{-13}$.
Moreover, eigendecomposition does not guarantee all class embeddings to lie in the positive orthant of the feature space.
However, we have found this to be a beneficial regularization in practice, resulting in slightly better performance.

On the other hand, when dealing with a large number of classes, the eigendecomposition can be useful to obtain a low-dimensional embedding that does not reconstruct $\mathcal{S}$ exactly but approximates it as best as possible by keeping only the eigenvectors corresponding to the largest eigenvalues.
However, the resulting class embeddings will not be L2-normalized any more.
Experiments in \cref{app:low-dimensional} show that our method can still provide superior retrieval performance even with very low-dimensional embeddings, which is also advantageous regarding memory consumption when dealing with large datasets.

\section{Mapping Images onto Class Centroids}
\label{sec:learning-embeddings}

\noindent
Knowing the target embeddings for all classes, we need to find a transformation $\psi: \mathfrak{X} \rightarrow \mathbb{R}^n$ from the space $\mathfrak{X}$ of images into the hierarchy-based semantic embedding space, so that image features are close to the centroid of their class.

For modeling $\psi$, we employ convolutional neural networks (CNNs) whose last layer has $n$ output channels and no activation function.
Since we embed all classes onto the unit hypersphere, the last layer is followed by L2 normalization of the feature vectors.
The network is trained on batches $B \in (\mathfrak{X} \times \mathcal{Y})^m$ of $m$ images $I_b \in \mathfrak{X}$ with class labels $y_b \in \mathcal{Y} = \{1,\dotsc,n\}$, $b = 1, \dotsc, m$, under the supervision of a simple loss function $\mathcal{L}_\mathrm{CORR}: (\mathfrak{X} \times \mathcal{Y})^m \rightarrow \mathbb{R}$ that enforces similarity between the learned image representations and the semantic embedding of their class:
\begin{equation}
    \label{eq:corr-loss}
    \mathcal{L}_\mathrm{CORR}(B) =
    \frac{1}{m} \sum_{b=1}^m \left( 1 - \psi(I_b)^\top \varphi(c_{y_b}) \right) \;.
\end{equation}

Note that the class centroids $\varphi(c_1), \dotsc, \varphi(c_n)$ are computed beforehand using the algorithm described in \cref{sec:class-embeddings} and fixed during the training of the network.
This allows our loss function to be used stand-alone, as opposed to, for example, the center loss \cite{wen2016centerloss}, which requires additional supervision.
Moreover, the loss caused by each sample is independent from all other samples, so that expensive mining of hard pairs \cite{chopra2005contrastiveloss} or triplets \cite{schroff2015facenet} of samples is not necessary.

Features learned this way may not only be used for retrieval but also for classification by assigning a sample to the class whose embedding is closest in the feature space.
However, one might as well add a fully-connected layer with softmax activation on top of the embedding layer, producing output $\rho: \mathfrak{X} \rightarrow [0,1]^{n}$.
The network could then simultaneously be trained for computing semantic image embeddings and classifying images using a combined loss
\begin{equation}
    \label{eq:embed-cls-loss}
    \mathcal{L}_\mathrm{CORR+CLS} = \mathcal{L}_\mathrm{CORR} + \lambda \cdot \mathcal{L}_\mathrm{CLS} \;,
\end{equation}
where $\mathcal{L}_\mathrm{CLS}: (\mathfrak{X} \times \mathcal{Y})^m \rightarrow \mathbb{R}$ denotes the categorical cross-entropy loss function
\begin{equation}
    \label{eq:crossent-loss}
    \mathcal{L}_\mathrm{CLS}(B) =
    \frac{1}{m} \sum_{b=1}^m \log(\rho(I_b)_{y_b}) \;.
\end{equation}

Since we would like the embedding loss $\mathcal{L}_\mathrm{CORR}$ to dominate the learning of image representations, we set $\lambda$ to a small value of $\lambda = 0.1$ in our experiments.

\section{Experiments}
\label{sec:experiments}

\noindent
In the following, we present results on three different datasets and compare our method for learning semantic image representations with features learned by other methods that try to take relationships among classes into account.

\subsection{Datasets and Setups}
\label{subsec:datasets}

\subsubsection{CIFAR-100}
\label{par:cifar100}

The extensively benchmarked CIFAR-100 dataset \cite{krizhevsky2009cifar} consists of 100 classes with 500 training and 100 test images each.
In order to make our approach applicable to this dataset, we created a taxonomy for the set of classes, mainly based on the WordNet ontology \cite{fellbaum1998wordnet} but slightly simplified and with a strict tree structure.
A visualization of that taxonomy can be found in \cref{app:cifar-hierarchy}.

We test our approach with 3 different architectures designed for this dataset \cite{krizhevsky2009cifar}, while, in general, any neural network architecture can be used for mapping images onto the semantic space of class embeddings:
\begin{itemize}
    
    \item Plain-11 \cite{barz2018deep}, a recently proposed strictly sequential, shallow, wide, VGG-like network consisting of only 11 trainable layers, which has been found to achieve classification performance competitive to ResNet-110 when trained using cyclical learning rates.
    
    \item A variant of ResNet-110 \cite{he2016resnet}, a deep residual network with 110 trainable layers. We use 32, 64, and 128 instead of 16, 32, and 64 channels per block, so that the number of features before the final fully-connected layer is greater than the number of classes. To make this difference obvious, we refer to this architecture as ``ResNet-110w'' in the following (``w'' for ``wide'').
    
    \item PyramidNet-272-200 \cite{han2017pyramidnet}, a deep residual network whose number of channels increases with every layer and not just after pooling.
    
\end{itemize}

\noindent
Following \cite{barz2018deep}, we train these networks using SGD with warm restarts (SGDR \cite{loshchilov2016sgdr}), starting with a base learning rate of $0.1$ and smoothly decreasing it over 12 epochs to a minimum of $10^{-6}$ using cosine annealing.
The next cycle then begins with the base learning rate and the length of the cycles is doubled at the end of each one.
All architectures are trained over a total number of 5 cycles ($= 372$ epochs) using a batch-size of $100$.
To prevent divergence caused by the initially high learning rate, we use gradient clipping \cite{pascanu2013gradientclipping} and restrict the norm of gradients to a maximum of $10.0$.

\subsubsection{North American Birds}
\label{par:nab}

The North American Birds (NAB) dataset\footnote{\url{http://dl.allaboutbirds.org/nabirds}} comprises 23,929 training and 24,633 test images showing birds from 555 different species.
A 5 levels deep hierarchy of those classes is provided with the dataset.

Due to the comparatively small amount of training data, we do not only report results for training from scratch on the NAB dataset only, but also for fine-tuning models pre-trained on ILSVRC 2012 (see below).

During training, the images are randomly resized so that their smaller side is between 256 and 480 pixels wide and a random crop of size $224 \times 224$ pixels is extracted.
Moreover, we apply random horizontal flipping and random erasing \cite{zhong2017random} for data augmentation.

We use the ResNet-50 architecture and train it for 4 cycles of SGDR (180 epochs) as described above, using a base learning rate of $0.5$ and a batch size of 128 images.
Only in case of our method with the $\mathcal{L}_\mathrm{CORR}$ loss (\ie, without classifcation objective), one additional SGDR cycle was required to achieve convergence (\ie, 372 epochs).
For fine-tuning, we did not change the learning rate schedule for our semantic embeddings, but chose the number of cycles and the learning rate individually for all competitors to achieve good performance without overfitting.

\subsubsection{ILSVRC 2012}
\label{par:ilsvrc}

We also conduct experiments on data from the ImageNet Large Scale Visual Recognition Challenge (ILSVRC) 2012, which comprises over 1.1 million training and 50,000 test images from 1,000 classes.
The classes originate from the WordNet ontology.
Since this taxonomy is not a tree, we used the method of Redmon \& Farhadi \cite{redmon2016yolo9000} described in \cref{subsec:semantic-similarity} for deriving a tree-shaped hierarchy from WordNet.
For the evaluation, however, we used the full WordNet taxonomy for computing class similarities.

We employ the ResNet-50 architecture \cite{he2016resnet} with a batch size of 128 images but only use a single cycle of SGDR, \ie, just the cosine annealing without warm restarts, starting with a learning rate of $0.1$ and annealing it down to $10^{-6}$ over 80 epochs.
We used the same data augmentation techniques as He \etal \cite{he2016resnet} (flipping, scale augmentation, random cropping), except color augmentation, and use the weights provided by them as classification-based baseline to compare our method with.

\begin{figure*}[t]
    \includegraphics[width=\linewidth]{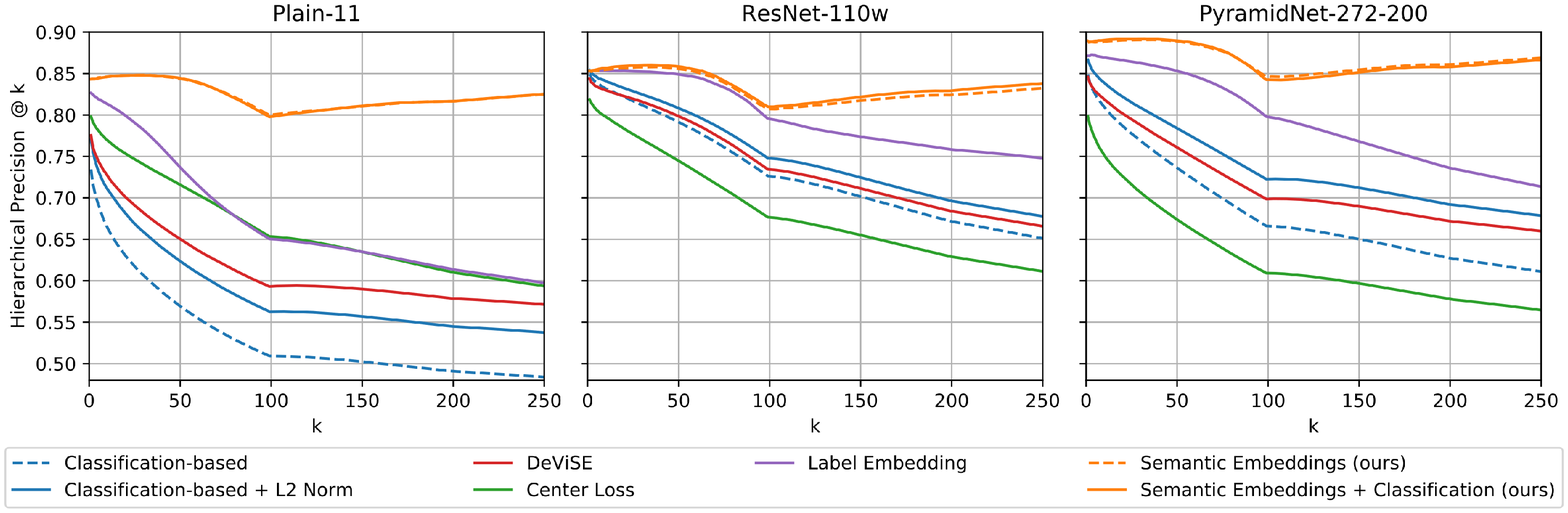}
    \caption{Hierarchical precision on CIFAR-100 at several cutoff points k for various network architectures.}
    \label{fig:hierarchical-precision}
\end{figure*}

\subsection{Performance Metrics}
\label{par:perf-metrics}

\noindent
Image retrieval tasks are often evaluated using the precision of the top k results (P@k) and mean average precision (mAP).
However, these metrics do not take similarities among classes and varying misclassification costs into account. 
Thus, we introduce variants of them that are aware of the semantic relationships among classes.

Let $x_q \in \mathfrak{X}$ be a query image belonging to class $y_q \in \mathcal{Y} = \{1, \dotsc, n\}$ and $R = \bigl((x_1, y_1), \dotsc, (x_m, y_m)\bigr)$ denote an ordered list of retrieved images $x_i \in \mathfrak{X}$ and their associated classes $y_i \in \mathcal{Y}$.
Following Deng et al. \cite{deng2011hierarchical}, we define the \textit{hierarchical precision at k (HP@k)} for $1 \leq k \leq m$ as
\begin{equation}
    \label{eq:hierarchical-precision}
    \mathrm{HP@k}(R) = \frac{
        \sum_{i=1}^k \semsim(y_q, y_i)
    }{
        \max_\pi \sum_{i=1}^k \semsim(y_q, y_{\pi_i})
    } \;,
\end{equation}
where $\pi$ denotes any permutation of the indices from $1$ to $m$. Thus, the denominator in \cref{eq:hierarchical-precision} normalizes the hierarchical precision through division by the sum of the precisions of the best possible ranking.

For several values of k, HP@k can be plotted as a curve for gaining a more detailed insight into the behavior of a retrieval system regarding the top few results (short-term performance) and the overall ranking (long-term performance).
We denote the area under that curve ranging from $k=1$ to $k=K$ as \textit{average hierarchical precision at $K$ (AHP@K)}.
It can be used to compare the semantic performance of image retrieval systems by means of a single number.
In the following, we always report the AHP@250, because we do not expect the typical user to inspect more than 250 results.

\subsection{Competitors}
\label{par:competitors}

\noindent
We evaluate two variants of our hierarchy-based semantic image embeddings: trained with $\mathcal{L}_\mathrm{CORR}$ only (eq.\ \cref{eq:corr-loss}) or trained with the combination of correlation and categorical cross-entropy loss, $\mathcal{L}_\mathrm{CORR+CLS}$, according to \cref{eq:embed-cls-loss}.

As a baseline, we compare them with features extracted from the last layer right before the final classification layer of the same network architecture, but trained solely for classification.
We also evaluate the performance of L2-normalized versions of these classification-based features, which are usually used for CBIR (\cf \cref{sec:introduction}).

Moreover, we compare with DeViSE \cite{frome2013devise}, the center loss \cite{wen2016centerloss}, and label embedding networks \cite{sun2017label}.
All methods have been applied to identical network architectures and trained with exactly the same optimization strategy, except DeViSE, which requires special training.

Since DeViSE \cite{frome2013devise} learns to map images onto word embeddings learned from text corpora, we were only able to apply it on CIFAR-100, since about one third of the 1,000 ILSVRC classes could not be matched to word embeddings automatically and the scientific bird names of the NAB dataset are not part of the vocabulary either.
For CIFAR-100, however, we used 100-dimensional pre-computed GloVe word embeddings \cite{pennington2014glove} learned from Wikipedia.

\subsection{Semantic Image Retrieval Performance}
\label{subsec:retrieval-performance}

\begin{table*}[t]
    \setlength\tabcolsep{1mm}
    \begin{tabularx}{\linewidth}{X @{\hspace{4mm}} ccc @{\hspace{4mm}} cc @{\hspace{4mm}} c}
        \toprule
        \multirow{2}{*}[-3pt]{Method} & \multicolumn{3}{c@{\hspace{4mm}}}{CIFAR-100} & \multicolumn{2}{c@{\hspace{4mm}}}{NAB} & \multirow{2}{*}[-3pt]{ILSVRC} \\
        \cmidrule(r{4mm}){2-4} \cmidrule(r{4mm}){5-6} & Plain-11 & ResNet-110w & PyramidNet & from scratch & fine-tuned & \\
        \midrule
        Classification-based & 0.5287 & 0.7261 & 0.6775 & 0.2538 & 0.5988 & 0.6831 \\
        Classification-based + L2 Norm & 0.5821 & 0.7468 & 0.7334 & 0.2696 & 0.6259 & 0.7132 \\
        DeViSE \cite{frome2013devise} & 0.6117 & 0.7348 & 0.7116 & --- & --- & --- \\
        Center Loss \cite{wen2016centerloss} & 0.6587 & 0.6815 & 0.6227 & 0.4185 & 0.4916 & 0.4094 \\
        Label Embedding \cite{sun2017label} & 0.6678 & 0.7950 & 0.7888 & 0.4251 & 0.5735 & 0.4769 \\
        \midrule
        Semantic Embeddings ($\mathcal{L}_\mathrm{CORR}$) [ours] & \textbf{0.8207} & 0.8290 & \textbf{0.8653} & 0.7157 & 0.7849 & 0.7902 \\
        Semantic Embeddings ($\mathcal{L}_\mathrm{CORR+CLS}$) [ours] & 0.8205 & \textbf{0.8329} & 0.8638 & \textbf{0.7399} & \textbf{0.8146} & \textbf{0.8242} \\
        \bottomrule
    \end{tabularx}
    \caption{Retrieval performance of different image features in mAHP@250. The best value per column is set in bold font.}
    \label{tbl:retrieval-performance}
\end{table*}

\begin{table*}[t]
    \setlength\tabcolsep{1mm}
    \begin{tabularx}{\linewidth}{X @{\hspace{4mm}} ccc @{\hspace{4mm}} cc @{\hspace{4mm}} c}
        \toprule
        \multirow{2}{*}[-3pt]{Method} & \multicolumn{3}{c@{\hspace{4mm}}}{CIFAR-100} & \multicolumn{2}{c@{\hspace{4mm}}}{NAB} & \multirow{2}{*}[-3pt]{ILSVRC} \\
        \cmidrule(r{4mm}){2-4} \cmidrule(r{4mm}){5-6} & Plain-11 & ResNet-110w & PyramidNet & from scratch & fine-tuned & \\
        \midrule
        Classification-based & 72.18\% & 76.95\% & \textbf{81.44\%} & 33.35\% & \textbf{70.17\%} & \textbf{73.87\%} \\
        DeViSE \cite{frome2013devise} & 69.24\% & 74.66\% & 77.32\% & --- & --- & --- \\
        Center Loss \cite{wen2016centerloss} & 73.27\% & 75.18\% & 76.83\% & 56.07\% & 61.72\% & 70.05\% \\
        Label Embedding \cite{sun2017label} & \textbf{74.38\%} & \textbf{76.96\%} & 79.35\% & 48.05\% & 61.87\% & 70.94\% \\
        \midrule
        Semantic Embeddings ($\mathcal{L}_\mathrm{CORR}$) [ours] & 73.99\% & 75.03\% & 79.87\% & 57.52\% & 63.05\% & 48.97\% \\
        Semantic Embeddings ($\mathcal{L}_\mathrm{CORR+CLS}$) [ours] & 74.10\% & 76.60\% & 80.49\% & \textbf{59.46\%} & 69.49\% & 69.18\% \\
        \bottomrule
    \end{tabularx}
    \caption{Balanced classification accuracy of various methods. The best value in every column is set in bold font.}
    \label{tbl:classification-performance}
\end{table*}

\noindent
For all datasets, we used each of the test images as individual query, aiming to retrieve semantically similar images from the remaining ones.
Retrieval is performed by ranking the images in the database decreasingly according to their dot product with the query image in the feature space.

The mean values of HP@k on CIFAR-100 over all queries at the first 250 values of k are reported in \cref{fig:hierarchical-precision}.
It can clearly be seen that our hierarchy-based semantic embeddings achieve a much higher hierarchical precision than all competing methods.
While the differences are rather small when considering only very few top-scoring results, our method maintains a much higher precision when taking more retrieved images into account. 

There is an interesting turning point in the precision curves at $k=100$, where only our method suddenly starts increasing hierarchical precision.
This is because there are exactly 100 test images per class, so that retrieving only images from exactly the same category as the query is not sufficient any more after this point.
Instead, a semantically-aware retrieval method should continue retrieving images from the most similar classes at the later positions in the ranking, of which only our method is capable.
It is possible for the HP@k to improve after this point because previously missed good matches can be retrieved that result in a larger increase of the nominator of \cref{eq:hierarchical-precision} than the denominator.

Additionally taking a classification objective into account during training improves the semantic retrieval performance on CIFAR-100 only slightly compared to using our simple $\mathcal{L}_\mathrm{CORR}$ loss function alone.
However, it is clearly beneficial on NAB and ILSVRC.
This can be seen from the mAHP@250 reported for all datasets in \cref{tbl:retrieval-performance}.
Our approach outperforms the second-best method for every dataset and network architecture by between 5\% and 23\% relative improvement of mAHP@250 on CIFAR-100, 36\%-74\% on NAB, and 16\% on ILSVRC.


We also found that our approach outperforms all competitors in terms of classical mAP, which does not take class similarities into account.
The detailed results as well as qualitative examples for ILSVRC are provided in \cref{app:further-results,app:ilsvrc-examples}.
Meanwhile, \cref{fig:retrieval_comparison} shows some qualitative results on CIFAR-100. 

\subsection{Classification Performance}
\label{subsec:classification-performance}

\noindent
While our method achieves superior performance in the scenario of content-based image retrieval, we also wanted to make sure that it does not sacrifice classification performance.
Though classification was not the objective of our work, we have hence also compared the balanced classification accuracy obtained by all tested methods on all datasets and show the results in \cref{tbl:classification-performance}.
It can be seen that learning to map images onto the semantic space of class embeddings does not impair classification performance unreasonably.

For DeViSE, which does not produce any class predictions, we have trained a linear SVM on top of the extracted features.
The same issue arises for hierarchy-based semantic embeddings trained with $\mathcal{L}_\mathrm{CORR}$, where we performed classification by assigning images to the class with the nearest embedding in the semantic feature space.
Though doing so gives fair results, the accuracy becomes more competitive when training with $\mathcal{L}_\mathrm{CORR+CLS}$ instead, \ie, a combination of the embedding and the classification objective.

In the case of the NAB dataset, our method even obtains the best accuracy when training from scratch.
Incorporating prior knowledge about class similarities obviously facilitates learning on a limited amount of training data.

\section{Conclusions and Future Work}
\label{sec:conclusions}

\noindent
We proposed a novel method for integrating basic prior knowledge about the semantic relationships between classes, given as class taxonomy, into deep learning.
Our hierarchy-based semantic embeddings preserve the semantic similarity of the classes in the joint space of image and class embeddings and thus allow for retrieving images from a database that are not only \textit{visually}, but also \textit{semantically} similar to a given query image.
This avoids unrelated matches and improves the quality of content-based image retrieval results significantly compared with other recent representation learning methods.

In contrast to other often used class representations such as text-based word embeddings, the hierarchy-based embedding space constructed by our method allows for straightforward training of neural networks by learning a regression of the class centroids using a simple loss function involving only a dot product.

The learned image features have also proven to be suitable for image classification, providing performance similar to that of networks trained explicitly for that task only.

Since the semantic target feature space is, by design, very specific to the classes present in the training set, generalization w.r.t.\ novel classes is still an issue.
It thus seems promising to investigate the use of activations at earlier layers in the network, which we expect to be more general.

\ifwacvfinal
\subsubsection*{Acknowledgements}

\begin{footnotesize}
\noindent
This work was supported by the German Research Foundation as part of the
priority programme ``Volunteered Geographic Information: Interpretation,
Visualisation and Social Computing'' (SPP 1894, contract number DE 735/11-1).
We also gratefully acknowledge the support of NVIDIA Corporation with the
donation of Titan Xp GPUs used for this research.\par
\end{footnotesize}
\fi

{\small
\bibliographystyle{ieee}
\bibliography{references}
}

\definecolor{Gray}{gray}{0.9}
\definecolor{LightGray}{gray}{0.96}
\definecolor{DarkGray}{gray}{0.3}

\definecolor{sim000}{RGB}{165,   0,  38}
\definecolor{sim005}{RGB}{190,  24,  38}
\definecolor{sim011}{RGB}{215,  49,  39}
\definecolor{sim016}{RGB}{231,  82,  54}
\definecolor{sim021}{RGB}{244, 114,  69}
\definecolor{sim026}{RGB}{249, 149,  85}
\definecolor{sim032}{RGB}{253, 180, 102}
\definecolor{sim037}{RGB}{253, 208, 125}
\definecolor{sim042}{RGB}{254, 230, 149}
\definecolor{sim047}{RGB}{254, 247, 177}
\definecolor{sim053}{RGB}{245, 250, 177}
\definecolor{sim058}{RGB}{224, 242, 149}
\definecolor{sim063}{RGB}{201, 232, 128}
\definecolor{sim068}{RGB}{173, 220, 110}
\definecolor{sim074}{RGB}{142, 206, 103}
\definecolor{sim079}{RGB}{107, 191,  99}
\definecolor{sim084}{RGB}{ 69, 173,  90}
\definecolor{sim090}{RGB}{ 27, 152,  80}
\definecolor{sim095}{RGB}{ 13, 128,  67}
\definecolor{sim100}{RGB}{  0, 104,  55}

\onecolumn
\begin{appendices}
\crefalias{section}{appsec}

\section{$\semdist$ applied to trees is a metric}
\label{app:metric-proof}

\begin{theorem}
    Let $\mathcal{G} = (V, E)$ be a directed acyclic graph whose egdes $E \subseteq V \times V$ define a hyponymy relation between the semantic concepts in $V$.
    Furthermore, let $\mathcal{G}$ have exactly one unique root node $\rootnode(\mathcal{G})$ with indegree $\deg^-(\rootnode(\mathcal{G})) = 0$.
    The lowest common subsumer $\lcs(u, v)$ of two concepts $u, v \in V$ hence always exists.
    Moreover, let $\treeheight(u)$ denote the maximum length of a path from $u \in V$ to a leaf node, and $\mathcal{C} \subseteq V$ be a set of classes of interest.
    
    Then, the semantic dissimilarity $\semdist: \mathcal{C} \times \mathcal{C} \rightarrow \mathbb{R}$ between classes given by
    \begin{equation}
    \semdist(u, v) = \frac{\treeheight(\lcs(u, v))}{\max_{w \in V} \treeheight(w)}
    \end{equation}
    is a proper metric if
    \begin{enumerate}[label=(\alph*), ref=(\alph*)]
        \item\label{item:tree} $\mathcal{G}$ is a tree, \ie, all nodes $u \in V \setminus \{\rootnode(\mathcal{G})\}$ have indegree $\deg^-(u) = 1$, and
        \item\label{item:leafs} all classes of interest are leaf nodes of the hierarchy, \ie, all $u \in \mathcal{C}$ have outdegree $\deg^+(u) = 0$.
    \end{enumerate}
\end{theorem}

\begin{proof}
    For being a proper metric, $\semdist$ must possess the following properties:
    \begin{enumerate}[label=(\roman*), ref=(\roman*)]
        \item\label{item:nonneg} Non-negativity: $\semdist(u, v) \geq 0$.
        \item\label{item:sym} Symmetry: $\semdist(u, v) = \semdist(v, u)$.
        \item\label{item:ident} Identity of indiscernibles: $\semdist(u, v) = 0 \Leftrightarrow u = v$.
        \item\label{item:tria} Triangle inequality: $\semdist(u, w) \leq \semdist(u, v) + \semdist(v, w)$.
    \end{enumerate}
    
    The conditions \ref{item:nonneg} and \ref{item:sym} are always satisfied since $\treeheight: V \rightarrow \mathbb{R}$ is defined as the length of a path, which cannot be negative, and the lowest common subsumer (LCS) of two nodes is independent of the order of arguments.
    
    The proof with respect to the remaining properties \ref{item:ident} and \ref{item:tria} can be conducted as follows:
    
    \begin{description}[itemsep=1em]
        
        \item[\ref{item:leafs}$\rightarrow$\ref{item:ident}]
        Let $u, v \in \mathcal{C}$ be two classes with $\semdist(u, v) = 0$. This means that their LCS has height 0 and hence must be a leaf node. Because leaf nodes have, by definition, no further children, $u = \lcs(u, v) = v$.
        On the other hand, for any class $w \in \mathcal{C}$, $\semdist(w, w) = 0$ because $\lcs(w, w) = w$ and $w$ is a leaf node according to \ref{item:leafs}.
        
        \item[\ref{item:tree}$\rightarrow$\ref{item:tria}]
        Let $u,v,w \in \mathcal{C}$ be three classes.
        Due to \ref{item:tree}, there exists exactly one unique path from the root of the hierarchy to any node.
        Hence, $\lcs(u,v)$ and $\lcs(v,w)$ both lie on the path from $\rootnode(\mathcal{G})$ to $v$ and they are, thus, either identical or one is an ancestor of the other.
        Without loss of generality, we assume that $\lcs(u,v)$ is an ancestor of $\lcs(v,w)$ and thus lies on the root-paths to $u$, $v$, and $w$.
        In particular, $\lcs(u,v)$ is a subsumer of $u$ and $w$ and, therefore, $\treeheight(\lcs(u,w)) \leq \treeheight(\lcs(u,v))$.
        In general, it follows that $\semdist(u,w) \leq \max\{\semdist(u,v), \semdist(v,w)\} \leq \semdist(u,v) + \semdist(v,w)$, given the non-negativity of $\semdist$.
        
    \end{description}
    
\end{proof}

\subsection*{Remark regarding the inversion}

If $\semdist$ is a metric, all classes $u \in \mathcal{C}$ of interest must necessarily be leaf nodes, since $\semdist(u, u) = 0 \Rightarrow \treeheight(\lcs(u, u)) = \treeheight(u) = 0$.

However, \ref{item:tria}$\rightarrow$\ref{item:tree} does not hold in general, since $\semdist$ may even be a metric for graphs $\mathcal{G}$ that are not trees.
An example is given in \cref{subfig:nontree-metric}.
Nevertheless, most such graphs violate the triangle inequality, like the example shown in \cref{subfig:nontree-nonmetric}.

\begin{figure}
    \begin{subfigure}{.48\linewidth}
        \centering
        \includegraphics[height=7cm]{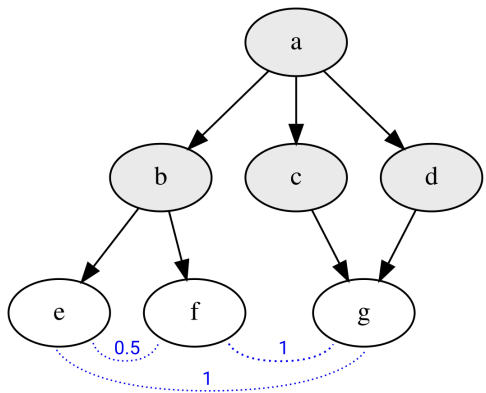}
        \caption{A non-tree hierarchy where $\semdist$ is a metric.}
        \label{subfig:nontree-metric}
    \end{subfigure}\hfil%
    \begin{subfigure}{.48\linewidth}
        \centering
        \includegraphics[height=7cm]{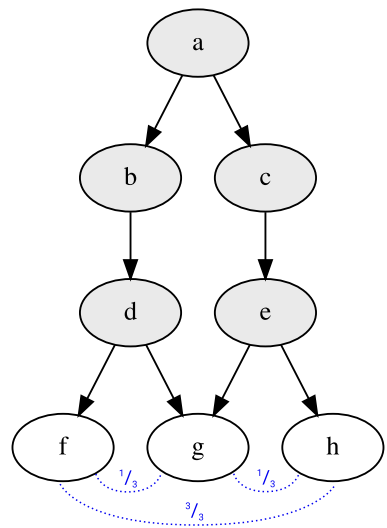}
        \caption{A non-tree hierarchy where $\semdist$ violates the triangle inequality.}
        \label{subfig:nontree-nonmetric}
    \end{subfigure}
    \caption{Examples for non-tree hierarchies.}
    \label{fig:nontree}
\end{figure}

\clearpage

\section{Further Quantitative Results}
\label{app:further-results}

\begin{figure}[h]
    \centering
    \includegraphics[width=\linewidth]{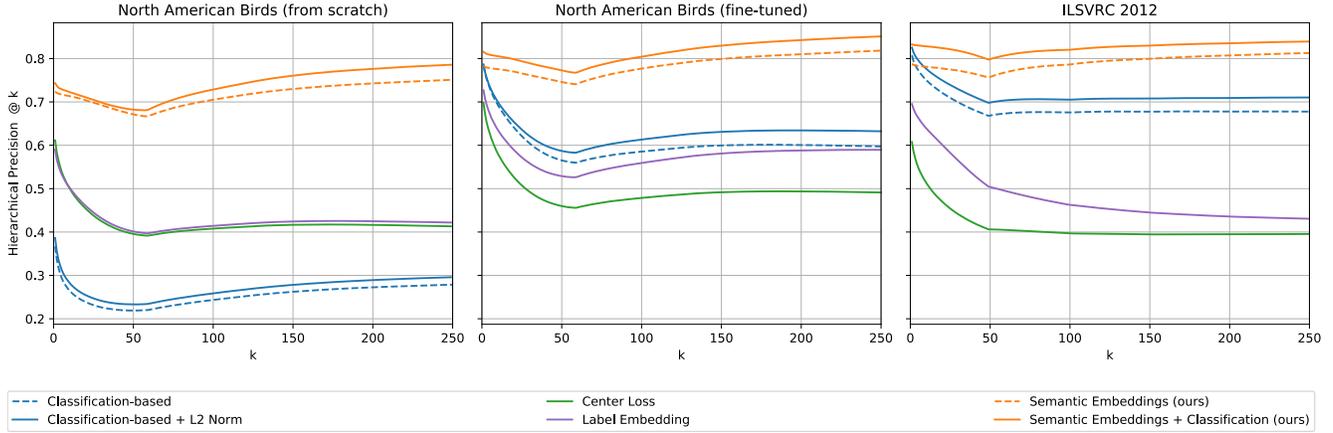}
    \caption{Hierarchical precision on NAB and ILSVRC 2012.}
    \label{fig:nab-ilsvrc-performance}
\end{figure}

\begin{table*}[h]
    \setlength\tabcolsep{1mm}
    \begin{tabularx}{\linewidth}{X @{\hspace{4mm}} ccc @{\hspace{4mm}} cc @{\hspace{4mm}} c}
        \toprule
        \multirow{2}{*}[-3pt]{Method} & \multicolumn{3}{c@{\hspace{4mm}}}{CIFAR-100} & \multicolumn{2}{c@{\hspace{4mm}}}{NAB} & \multirow{2}{*}[-3pt]{ILSVRC} \\
        \cmidrule(r{4mm}){2-4} \cmidrule(r{4mm}){5-6} & Plain-11 & ResNet-110w & PyramidNet & from scratch & fine-tuned & \\
        \midrule
        Classification-based & 0.2078 & 0.4870 & 0.3643 & 0.0283 & 0.2771 & 0.2184 \\
        Classification-based + L2 Norm & 0.2666 & 0.5305 & 0.4621 & 0.0363 & 0.3194 & 0.2900 \\
        DeViSE & 0.2879 & 0.5016 & 0.4131 & --- & --- & --- \\
        Center Loss & 0.4180 & 0.4153 & 0.3029 & 0.1591 & 0.1802 & 0.1285 \\
        Label Embedding & 0.2747 & \textbf{0.6202} & 0.5920 & 0.1271 & 0.2417 & 0.2683 \\
        \midrule
        Semantic Embeddings ($\mathcal{L}_\mathrm{CORR}$) [ours] & 0.5660 & 0.5900 & 0.6642 & 0.4249 & 0.5246 & 0.3037 \\
        Semantic Embeddings ($\mathcal{L}_\mathrm{CORR+CLS}$) [ours] & \textbf{0.5886} & 0.6107 & \textbf{0.6808} & \textbf{0.4316} & \textbf{0.5768} & \textbf{0.4508} \\
        \bottomrule
    \end{tabularx}
    \caption{Classical mean average precision (mAP) on all datasets. The best value per column is set in bold font. Obviously, optimizing for a classification criterion only leads to sub-optimal features for image retrieval.}
    \label{tbl:classical-retrieval-performance}
\end{table*}

\clearpage

\section{Qualitative Results on ILSVRC 2012}
\label{app:ilsvrc-examples}

\begin{figure}[h]
    \includegraphics[width=\linewidth]{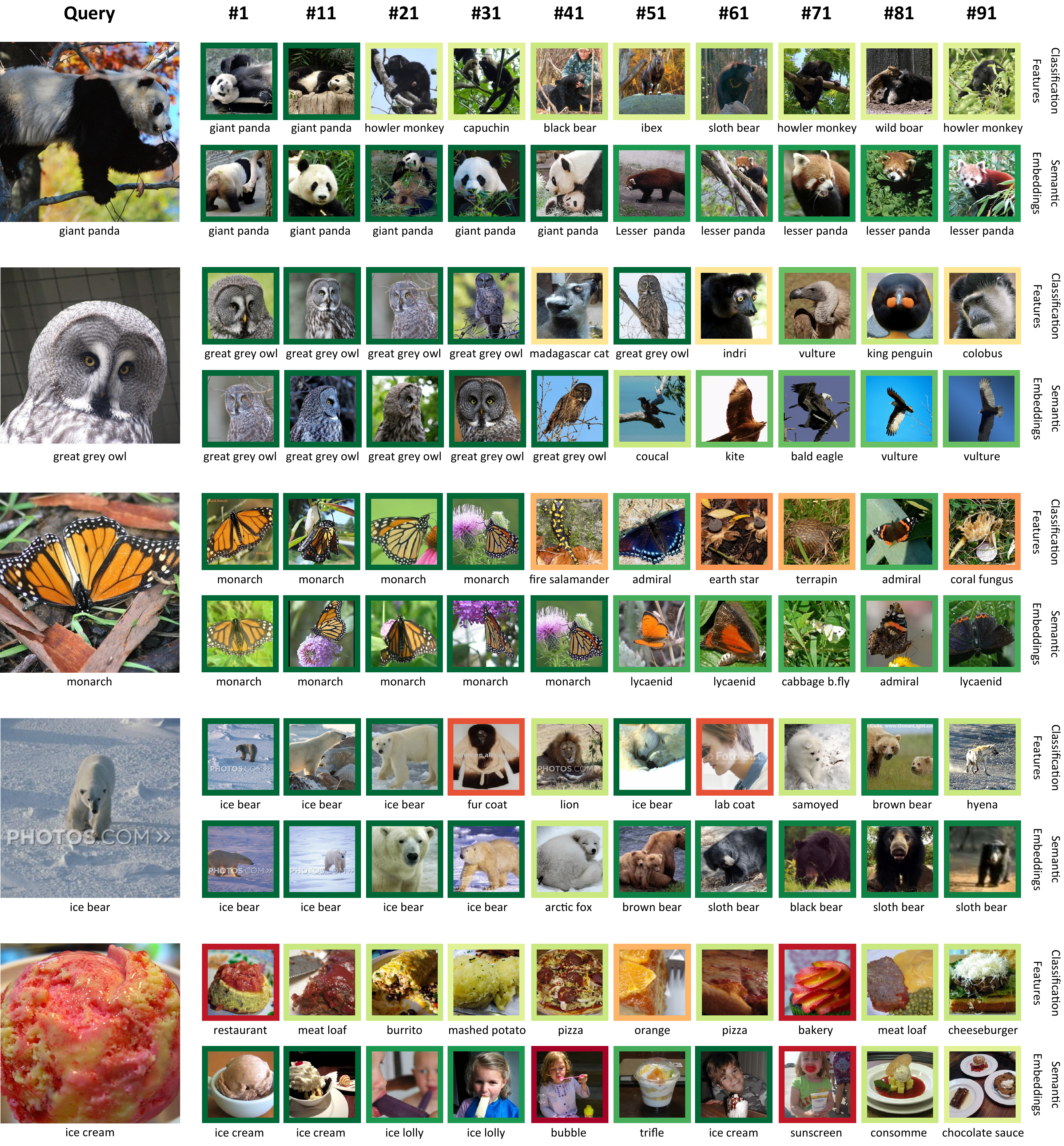}
    \caption{Comparison of a subset of the top 100 retrieval results using L2-normalized classification-based and our hierarchy-based semantic features for 3 exemplary queries on ILSVRC 2012. Image captions specify the ground-truth classes of the images and the border color encodes the semantic similarity of that class to the class of the query image, with dark green being most similar and dark red being most dissimilar.}
    \label{fig:imagenet-examples}
\end{figure}

\begin{figure}[p]
    \renewcommand{\arraystretch}{1.2}
    \setlength\tabcolsep{4pt}
    \begin{tabularx}{\linewidth}{ >{\columncolor{LightGray}}c >{\columncolor{Gray}}l >{\columncolor{Gray}\color{DarkGray}\scriptsize}r >{\columncolor{LightGray}}X >{\columncolor{LightGray}\color{DarkGray}\scriptsize}r }
        \textbf{Image} & \multicolumn{2}{>{\columncolor{Gray}}c}{\textbf{Classification-based}} & \multicolumn{2}{>{\columncolor{LightGray}}c}{\textbf{Semantic Embeddings (ours)}} \\
        \addlinespace
        
        & 1. \underline{Giant Panda} & (1.00) & 1. \underline{Giant Panda} & (1.00) \\
        & 2. American Black Bear & (0.63) & 2. Lesser Panda & (0.89) \\
        & 3. Ice Bear & (0.63) & 3. Colobus & (0.58) \\
        & 4. Gorilla & (0.58) & 4.  American Black Bear & (0.63) \\
        \multirow{-5}{*}{\includegraphics[height=2.31cm]{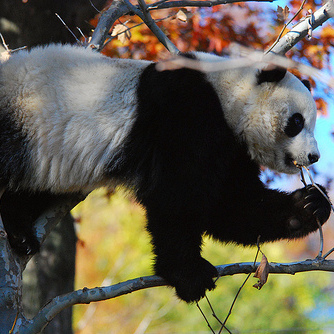}}
        & 5. Sloth Bear & (0.63) & 5. Guenon & (0.58) \\
        
        \addlinespace
        
        & 1. \underline{Great Grey Owl} & (1.00) & 1. \underline{Great Grey Owl} & (1.00) \\
        & 2. Sweatshirt & (0.16) & 2. Kite & (0.79) \\
        & 3. Bonnet & (0.16) & 3. Bald Eagle & (0.79) \\
        & 4. Guenon & (0.42) & 4. Vulture & (0.79) \\
        \multirow{-5}{*}{\includegraphics[height=2.31cm]{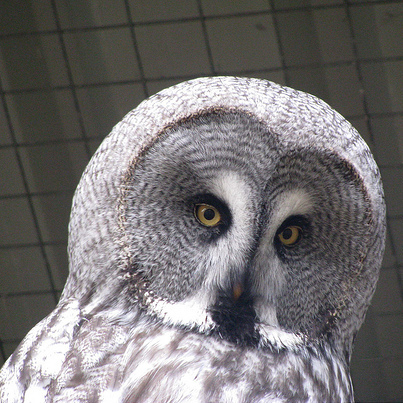}}
        & 5. African Grey & (0.63) & 5. Ruffed Grouse & (0.63) \\
        
        \addlinespace
        
        & 1. \underline{Monarch} & (1.00) & 1. \underline{Monarch} & (1.00) \\
        & 2. Earthstar & (0.26) & 2. Cabbage Butterfly & (0.84) \\
        & 3. Coral Fungus & (0.26) & 3. Admiral & (0.84) \\
        & 4. Stinkhorn & (0.26) & 4. Sulphur Butterfly & (0.84) \\
        \multirow{-5}{*}{\includegraphics[height=2.31cm]{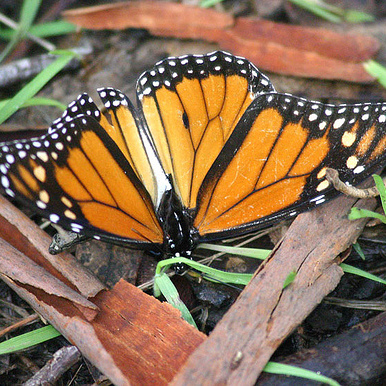}}
        & 5. Admiral & (0.84) & 5. Lycaenid & (0.84) \\
        
        \addlinespace
        
        & 1. \underline{Ice Bear} & (1.00) & 1. \underline{Ice Bear} & (1.00) \\
        & 2. Arctic Fox & (0.63) & 2. Brown Bear & (0.95) \\
        & 3. White Wolf & (0.63) & 3. Sloth Bear & (0.95) \\
        & 4. Samoyed & (0.63) & 4. Arctic Fox & (0.63) \\
        \multirow{-5}{*}{\includegraphics[height=2.31cm]{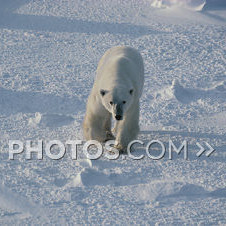}}
        & 5. Great Pyrenees & (0.63) & 5. American Black Bear & (0.95) \\
        
        \addlinespace
        
        & 1. \underline{Ice Cream} & (1.00) & 1. \underline{Ice Cream} & (1.00) \\
        & 2. Meat Loaf & (0.63) & 2. Ice Lolly & (0.84) \\
        & 3. Bakery & (0.05) & 3. Trifle & (0.89) \\
        & 4. Strawberry & (0.32) & 4. Chocolate Sauce & (0.58) \\
        \multirow{-5}{*}{\includegraphics[height=2.31cm]{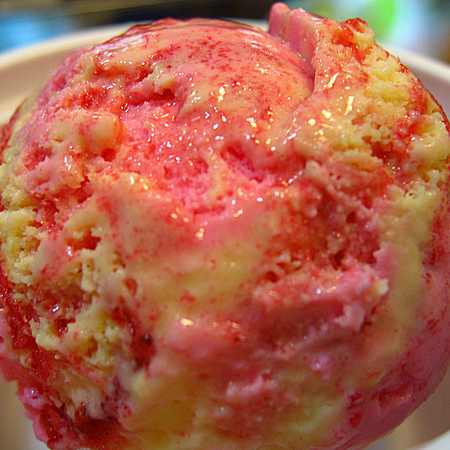}}
        & 5. Fig & (0.32) & 5. Plate & (0.79) \\
        
        \addlinespace
        
        & 1. \underline{Cocker Spaniel} & (1.00) & 1. \underline{Cocker Spaniel} & (1.00) \\
        & 2. Irish Setter & (0.84) & 2. Sussex Spaniel & (0.89) \\
        & 3. Sussex Spaniel & (0.89) & 3. Irish Setter & (0.84) \\
        & 4. Australien Terrier & (0.79) & 4. Welsh Springer Spaniel & (0.89) \\
        \multirow{-5}{*}{\includegraphics[height=2.31cm]{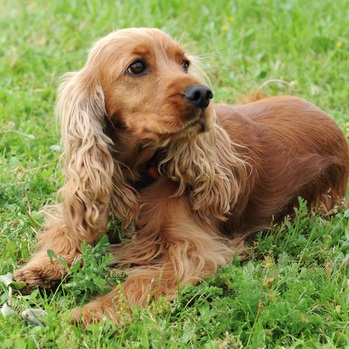}}
        & 5. Clumber & (0.89) & 5. Golden Retriever & (0.84) \\
    \end{tabularx}
    \caption{Top 5 classes predicted for several example images by a ResNet-50 trained purely for classification and by our network trained with $\mathcal{L}_\mathrm{CORR+CLS}$ incorporating semantic information. The correct label for each image is underlined and the numbers in parentheses specify the semantic similarity of the predicted class and the correct class. It can be seen that class predictions made based on our hierarchy-based semantic embeddings are much more relevant and consistent.}
    \label{fig:imagenet-class-predictions}
\end{figure}

\clearpage

\section{Low-dimensional Semantic Embeddings}
\label{app:low-dimensional}

\begin{figure}[h]
    \centering
    \includegraphics[width=\linewidth]{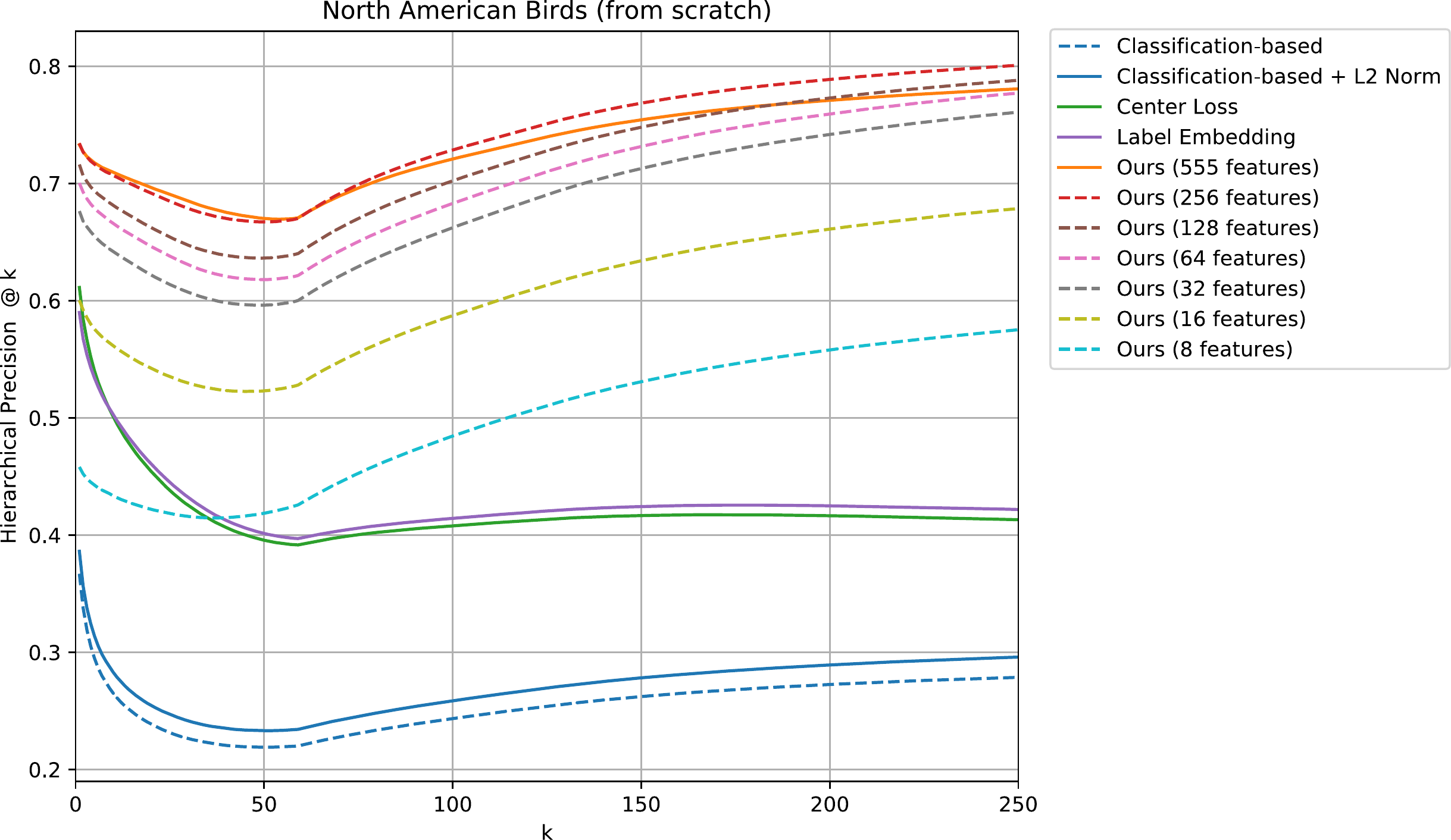}
    \caption{Hierarchical precision of our method for learning image representations based on class embeddings with varying dimensionality, compared with the usual baselines.}
    \label{fig:nab-embedding-dimensionality}
\end{figure}

As can be seen from the description of our algorithm for computing class embeddings in \cref{subsec:embedding-algorithm}, an embedding space with $n$ dimensions is required in general to find an embedding for $n$ classes that reproduces their semantic similarities exactly.
This can become problematic in settings with a high number of classes.

For such scenarios, we have proposed a method for computing low-dimensional embeddings of arbitrary dimensionality approximating the actual relationships among classes in \cref{subsec:low-dimensional-embeddings}.
paper.
We experimented with this possibility on the NAB dataset, learned from scratch, to see how reducing the number of features affects our algorithm for learning image representations and the semantic retrieval performance.

The results in \cref{fig:nab-embedding-dimensionality} show that obtaining low-dimensional class embeddings through eigendecomposition is a viable option for settings with a high number of classes.
Though the performance is worse than with the full amount of required features, our method still performs better than the competitors with as few as 16 features.
Our approach hence also allows obtaining very compact image descriptors, which is important when dealing with huge datasets.

Interestingly, the 256-dimensional approximation even gives slightly better results than the full embedding after the first 50 retrieved images.
We attribute this to the fact that fewer features leave less room for overfitting, so that slightly lower-dimensional embeddings can generalize better in this scenario.

\clearpage

\section{Taxonomy used for CIFAR-100}
\label{app:cifar-hierarchy}

\begin{center}
    \includegraphics[height=21cm]{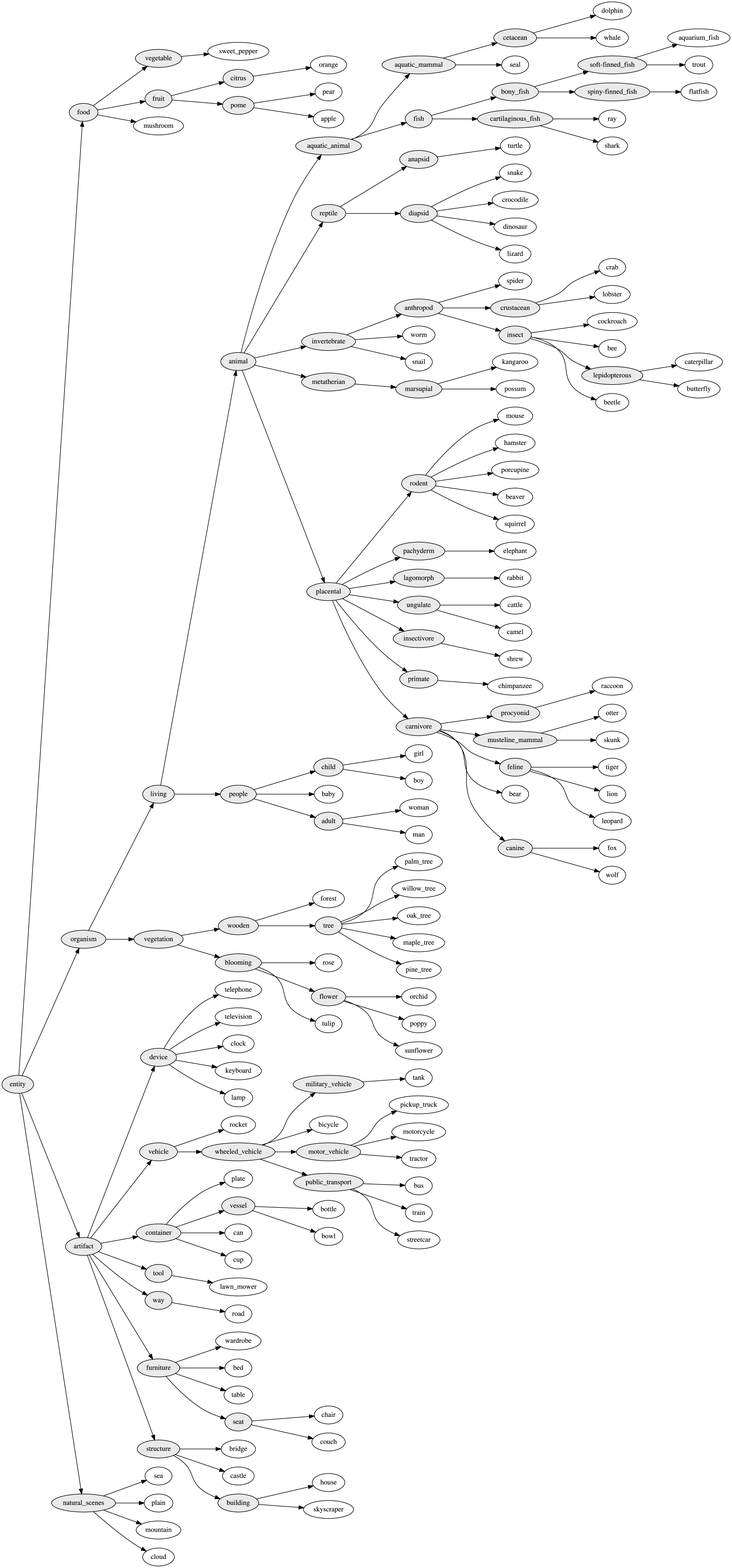}
\end{center}

\end{appendices}

\end{document}